\documentclass{article}

\usepackage{arxiv}
\usepackage[utf8]{inputenc} % allow utf-8 input
\usepackage[T1]{fontenc}    % use 8-bit T1 fonts
\usepackage{hyperref}       % hyperlinks
\usepackage{url}            % simple URL typesetting
\usepackage{booktabs}       % professional-quality tables
\usepackage{amsfonts}       % blackboard math symbols
\usepackage{nicefrac}       % compact symbols for 1/2, etc.
\usepackage{microtype}      % microtypography
\usepackage{graphicx}
\usepackage{natbib}
\usepackage{doi}
\usepackage{amsthm}
\usepackage{soul}
\usepackage{verbatim}
\usepackage{tabularx,booktabs}
\usepackage{array}
\usepackage[mathscr]{euscript}
\usepackage{blindtext}
\usepackage{titlesec}
\usepackage{enumitem}
\usepackage{xcolor}
\usepackage{mathtools}
\usepackage{array} 
\usepackage{subcaption} 
\usepackage{rotating} 

\usepackage{multirow, graphicx}
\usepackage{float}
  
\newtheorem{theorem}{Theorem}
\newtheorem{lemma}{Lemma}
\theoremstyle{definition}
\newtheorem{definition}{Definition}
\newtheorem{example}{Example}
\newtheorem{corollary}{Corollary}

\title{Unravelling the (In)compatibility of Statistical-Parity and Equalized-Odds}

\author{
	\href{https://orcid.org/0000-0001-5395-456X}{\includegraphics[scale=0.06]{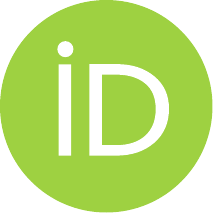}\hspace{1mm}Mortaza~S.~Bargh}\thanks{Corresponding author.} \\
	Research and Data Centre \\
	Dutch Ministry of Justice and Security\\
	The Hague, The Netherlands \\
	\texttt{m.shoae.bargh@wodc.nl} \\
	\And
	\href{https://orcid.org/0000-0003-2772-6330}{\includegraphics[scale=0.06]{orcid.pdf}\hspace{1mm}Sunil~Choenni} \\
	Research and Data Centre \\
	Dutch Ministry of Justice and Security\\
	The Hague, The Netherlands \\
	\texttt{r.choenni@wodc.nl} \\
	\And
	\href{https://orcid.org/0000-0003-2772-6330}{\includegraphics[scale=0.06]{orcid.pdf}\hspace{1mm}Floris~ter~Braak} \\
	Research and Data Centre \\
	Dutch Ministry of Justice and Security\\
	The Hague, The Netherlands \\
	\texttt{f.ter.braak@wodc.nl} \\
}

\renewcommand{\shorttitle}{Unravelling the (In)compatibility of Statistical-Parity and Equalized-Odds}

\hypersetup{
pdftitle={Unravelling the (In)compatibility of Statistical-Parity and Equalized-Odds},
pdfsubject={cs.LG},
pdfauthor={Mortaza~S.~Bargh, Sunil~Choenni, Floris~ter~Braak},
pdfkeywords={Algorithmic fairness, Base-rate imbalance, Equalized-odds, Incompatibility, Statistical-parity, Tradeoffs},
}

\begin{document}
\maketitle

\begin{abstract}
 A key challenge in employing data, algorithms and data-driven systems is to adhere to the principle of fairness and justice. Statistical fairness measures belong to an important category of technical/formal mechanisms for detecting fairness issues in data and algorithms. In this contribution we study the relations between two types of statistical fairness measures namely Statistical-Parity and Equalized-Odds. The Statistical-Parity measure does not rely on having ground truth, i.e., (objectively) labeled target attributes. This makes Statistical-Parity a suitable measure in practice for assessing fairness in data and data classification algorithms. Therefore, Statistical-Parity is adopted in many legal and professional frameworks for assessing algorithmic fairness. The Equalized-Odds measure, on the contrary, relies on having (reliable) ground-truth, which is not always feasible in practice. Nevertheless, there are several situations where the Equalized-Odds definition should be satisfied to enforce false prediction parity among sensitive social groups. We present a novel analyze of the relation between Statistical-Parity and Equalized-Odds, depending on the base-rates of sensitive groups. The analysis intuitively shows how and when base-rate imbalance causes incompatibility between Statistical-Parity and Equalized-Odds measures. As such, our approach provides insight in (how to make design) trade-offs between these measures in practice. Further, based on our results, we plea for examining base-rate (im)balance and investigating the possibility of such an incompatibility before enforcing or relying on the Statistical-Parity criterion. The insights provided, we foresee, may trigger initiatives to improve or adjust the current practice and/or the existing legal frameworks.
\end{abstract}


\section{Introduction}
\label{sec1}

Data are currently being generated, collected, shared, analyzed, and distributed at a fast-growing pace. Consequently, we witness a rising interest to harvest the increasingly available data by developing Information Systems (ISs) that ease our daily lives, create added values for businesses, provide insight into societal phenomena, and guide policymaking processes. Capitalizing on data, however, requires being attentive about the associated risks, like data often being blended with biased, partial, faulty, sensitive, and stigmatizing information about individuals, groups, and organizations. 
 
Dealing with wrongful bias and discrimination of individuals and groups (i.e., unfairness), which is embedded in collected data and/or induced due to inattentive data processing or data-driven service provisioning, is a key challenge in development and deployment of data-driven systems \cite{bibBalayn2021}. This challenge manifest not only in data processing algorithms $-$ like those used in Statistics, Artificial Intelligence (AI) and Machine Learning (ML) $-$ but also in various stages of data management like data selection, data mining, data cleaning, information integration and data discovery. These fairness issues must be identified and mitigated at various stages of data management and within data processing activities \cite{bibBalayn2021}. Not handling these fairness issues appropriately and adequately would harm individuals, groups, and the whole society at large; adversely affecting the basic human rights like privacy, liberty, autonomy, and dignity. 

Inspired by concepts like justice, bias and discrimination that stem from various disciplines such as ethics, philosophy, political science, legal science, criminology, sociology, anthropology, neuroscience, and psychology  \cite{bibDolata}, technical experts and IS developers have introduced algorithmic fairness definitions, measures and metrics to quantify (aspects of) those non-technical concepts and to detect or deal with unjustified biases and wrongful discriminations in datasets and algorithms. 

Statistical fairness measures and metrics belong to a category of formal fairness methods that have been developed and deployed widely. Statistical fairness measures/metrics provide automatically implementable mechanisms to detect and mitigate aspects of fairness issues in datasets and algorithms. Nevertheless, there are studies that show incompatibility among these measures which means that they cannot be satisfied simultaneously. For example, as shown in \cite{bibChouldechova, bibKleinberg}, Positive Predictive Value (PPV) equality measure is not compatible with Equalized-Odds measure, where the latter is based on False Negative Rate (FNR) and False Positive Rate (FPR) equality measures. This incompatibility was surfaced after fairness issues reported in the Correctional Offender Management Profiling for Alternative Sanctions (COMPAS) tool \cite{dieterich2016compas, propublicaCompass}. The COMPAS tool was used for measuring the risk of recidivism, i.e., a defender recommitting crime. Despite the tool being designed to be fair from Predictive-Parity perspective \cite{dieterich2016compas}, it was shown in a follow up study that the tool was unfair from the Equalized-Odds perspective \cite{propublicaCompass}. Specifically, African American defendants were more likely than Caucasian defendants to be incorrectly flagged as higher risk of recidivism (i.e., having a higher FPR), while Caucasian defendants were more likely than African American defendants to be incorrectly flagged as low risk of recidivism (i.e., having a higher FNR).

In this contribution we aim at providing an intuitive insight in how statistical fairness measures relate to each other, whether and in which circumstances they are (not) compatible with each other. Specifically, we consider two categories of statistical fairness measures namely Statistical-Parity and Equalized-Odds for datasets with binary classification data. The Statistical-Parity based definition of fairness is widely used in practice and legal frameworks to asses the fairness of classification predictions where there is no ground truth, i.e., (objectively) labeled target attributes, available. The Equalized-Odds based definition of fairness, on the contrary, can be used in those situations where there is ground truth available for such attributes. Specifically, our study focuses on the cases where there is a base-rate imbalance among sensitive social groups. Base-rate for a sensitive social group is the frequentist probability of observing positive outcomes for the group in a population/dataset. Within this context, our main contributions can be expressed as follows.
\begin{itemize}
\item  Using a novel approach, we show that enforcing both Statistical-Parity and Equalized-Odds requires either having base-rate balance, like the result reported in \cite{bibGarg2020}, or choosing for a random classifier.
\item The adopted approach allows us to graphically show how the base-rate imbalance impacts the compatibility of Statistical-Parity and Equalized-Odds adversely. 
\item The graphical representation allows us to intuitively show that an efficient classifier, which performs well in terms of, e.g., accuracy, requires making trade-off between Statistical-Parity and Equalized-Odds.
\item We recommend improvements to the current legal frameworks and practices which are based on using Statistical-Parity for assessing the fairness of classifiers and similar analytical approaches and/or decision-making processes.  
\end{itemize}

For this study, we formalize the problem context and adopt an analytical approach to derive the insights. Based on literature, we interpret the insights in their broader context, i.e., in terms of their legal and social implications. These insights shed light on the limitations of existing methods, practices and (legal) frameworks used for assessing the fairness of classification-based decision-making processes and suggest ways to remedy these shortcomings. As such they can be used for enriching the existing practices and legal frameworks. 

The outline of the paper is as follows. In Section \ref{sec2}, we provide background information about the problem context, notation used, and related work. In Section \ref{sec4}, we describe the problem context, which includes defining Statistical-Parity and Equalized-Odds measures formally. In Section \ref{sec5}, we employ our approach for analyzing the relation between these fairness definitions and, in Section \ref{sec6}, we discuss our results and their implications on current practices and legal frameworks. In Section \ref{sec7}, we draw our conclusions and mention several directions for future research. 

\section{Preliminaries}
\label{sec2}
In this section we provide some background information about the problem context in Section \ref{subsec21}, the notation used throughout the paper in Section \ref{subsec22}, and the related work in Section \ref{subsec23}. 
\subsection{Problem Context}
\label{subsec21}
The setting of this study concerns data-driven ISs where a data classification is performed along the journey path of data as part of information processing. Often in a classifier, training algorithms are used for deriving classification models, the trained models are used to predict outcomes based on some newly observed data, and the predicted outcomes are used as a (pre)selection criterion before and/or as (partial) evidence for making decisions about individuals (e.g., in credit granting) and/or groups (e.g., in policy-making). To illustrate various aspects of the the concepts discussed in this paper, we will use a running example. Note that this example, which will be extended in a few steps as we proceed, is purely hypothetical.      
\begin{example}
A bank receives an overwhelming number of applications for mortgage. As investigation of every application according to the bank's mortgage granting procedure is highly labor-intensive and time consuming, the bank makes a preselection of applications and investigates only the preselected applications according to the bank's procedure to decide on granting or denying the requested mortgages. To increase efficiency, the bank preselects potentially eligible applications, which correspond to those applicants who are assumed likely to pay back their mortgage in due time (e.g., in 10 years). Based on some historical data, therefore, the bank develops a classification model that uses some feature attributes like age group, education level, income category, and marital status to predict whether an applicant is able to pay back a mortgage in due time (e.g., the positive prediction outcome) or not (i.e., the negative prediction outcome). Those applicants who are marked positively, i.e., as potentially eligible applicants, by the classification algorithm are preselected. $\Box$
\end{example}

\subsection{Notation}
\label{subsec22}

In a classification algorithm, the values of feature attribute(s) $X$ are mapped to the values of the class attribute $\hat{Y}$ by using a trained classification model. In this way, every value of the class attribute $\hat{Y}$ is a prediction of the value of class attribute $Y$. For historical data, the values of the class attribute $Y$ might be known (and may have incorrectness/biases), which can be used for training and testing classification models. Ground truth and prediction attributes $Y$ and $\hat{Y}$ take values from the ground truth sample space $\mathscr{S}_Y$ and prediction sample space $\mathscr{S}_{\hat{Y}}$, respectively. 

In this contribution, we consider the case of binary classification where the outcomes of the classification is either \textit{desired} (also called a positive outcome or a success) or \textit{undesired} (also called a negative outcome or a failure). This terminology can be defined based on the purpose of a particular application in mind. For example, the desired outcome can be to obtain a mortgage, to diagnose a disease or to find a suitable employee if the data classification is used for granting mortgages, for diagnosing diseases, or for hiring people, respectively. We will use values $1$ and $0$ to denote the desired and undesired outcomes, respectively, for attributes $Y$ and $\hat{Y}$. In other words, $\mathscr{S}_Y = \mathscr{S}_{\hat{Y}} = \{0, 1\}$.

A binary channel, as shown in Figure \ref{figBasicEC}, can be used to model several performance measures of a binary classification. The binary channel model is specified by directed edges True Positives (TPs), False Positives (FPs), True Negatives (TNs), and False Negatives (FNs). Like in a binary communication channel, $\hat{Y}$ can be seen as a noisy version of $Y$. The noise exists due to uncertainty within the prediction activity, thus it can be referred to as prediction noise. 

\begin{figure}[!t]
\centering
\includegraphics[scale=0.25]{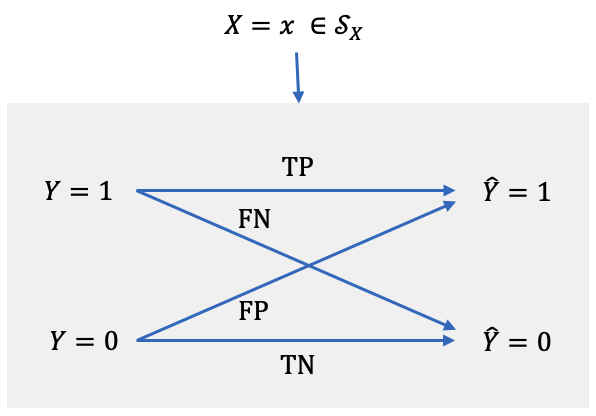}
\caption{A binary channel model of binary classification.}
\label{figBasicEC}
\end{figure}

As indicated in Figure \ref{figBasicEC}, a subset of feature attributes, called sensitive attribute(s) $S \subset X$, represent socially, legally and/or ethically sensitive groups among whom establishing some form of fairness is necessary. For example, features such as gender, race and religion, and their combination, are typical sensitive features for non-discrimination. Without loss of generality, in this paper we will assume that there is one sensitive attribute $S$, which takes binary values $1$ and $0$ for protected group (representing a socially unprivileged group) and unprotected group (representing a socially privileged group), respectively. The notation used in the binary channel model, which will be used also for describing the statistical fairness definitions and measures throughout the paper, is summarized in Table \ref{tabNot}.

\begin{table*}[!t]
\centering
\caption{Description of the attributes used in binary channel model in Figure \ref{figBasicEC}}
\label{tabNot}%
\begin{tabular}{@{}p{5em} p{0.85\textwidth}}
\toprule
Symbol 				&	Description 																																					\\
\midrule
$X$  					&	All feature attribute(s) describing some aspects of individuals. 																												\\
$S$					& 	The sensitive attribute, taking values $1$ and $0$ for protected and unprotected groups, respectively (i.e., sample space $\mathscr{S}_{S} = \{ 0, 1\}$). 									\\
$X' = X \backslash S$	& 	The non-sensitive attribute(s), describing some characteristics of individuals. 																									\\
$Y$					& 	The class attribute representing the ground truth, taking values $1$ and $0$ for desired and undesired outcomes, respectively (i.e., sample space $\mathscr{S}_Y = \{ 0, 1 \}$). 									\\
$\hat{Y}$ 				&	The predicted class attribute representing the outcomes of the binary classification, taking values $1$ and $0$ for desired and undesired outcomes, respectively (i.e., sample space $\mathscr{S}_Y = \{ 0, 1 \}$).			\\
\bottomrule
\end{tabular}
\end{table*}
								
\begin{example}
In the mortgage example, the feature attributes $X$ are the age group, education level, income category, and marital status; and the class attribute $Y$ and its prediction $\hat{Y}$ assume values 1 and 0, denoting a mortgagor being (predicted to be) able to pay back a mortgage or not, respectively. In this example, we consider the age group as the sensitive feature $S$ with two groups namely young workforce and experienced workforce, who have entered the housing market for a short or long time, respectively. The group of young workforce is assumed as the protected (i.e., unprivileged) group, for which $S=1$.  
$\Box$       
\end{example}

The measures that will be investigated in this contribution belong to a category of fairness measures called statistical fairness measures. Statistical fairness measures are quantified by frequentist probabilities that are, in turn, defined as fractions of the numbers of observed cases like $TP$s, $FP$s, $TN$s, $FN$s and their combinations. In our setting, these numbers, which are further specified per sensitive group $S=s$, where $s \in \{0, 1 \}$, are denoted by $N_{TP, S=s}$, $N_{FP, S=s}$, $N_{TN, S=s}$ and $N_{FN, S=s}$ as shown in Figure \ref{figBasicECs}. The total number of per sensitive group cases is $N_s = N_{S=s} = N_{TP, S=s} + N_{FP, S=s} + N_{TN, S=s} + N_{FN, S=s}$. Note that here we introduced and used $N_s$ instead of $N_{S=s}$ to simplify the notation. From now on, we will mainly use this simplified notation, i.e., use subscript $s$ (or its values 0 and 1) instead of subscript $S=s$ (or $S=0$ and $S=1$), wherever there is only one subscript and that subscript involves sensitive feature $S$. Total number of cases is $N_t = N_1 + N_0$. The total number of cases $N_t$ (or the number of per social group cases $N_s$) is the same at the input (in the ground truth space) and the output (in the prediction space) of the corresponding binary channel model. The fractions of sensitive group $S=s$ are $\pi_s = \frac{N_s}{N_t}$, where  $s \in \{0, 1 \}$. Note that $\pi_0 = 1-\pi_1$. The frequentist probabilities (or fractions) relevant to this study are summarized in Table \ref{tabParam} per social group $S=s$, where  $s \in \{0, 1 \}$.   

\begin{figure}[!t]
\centering
\includegraphics[scale=0.25]{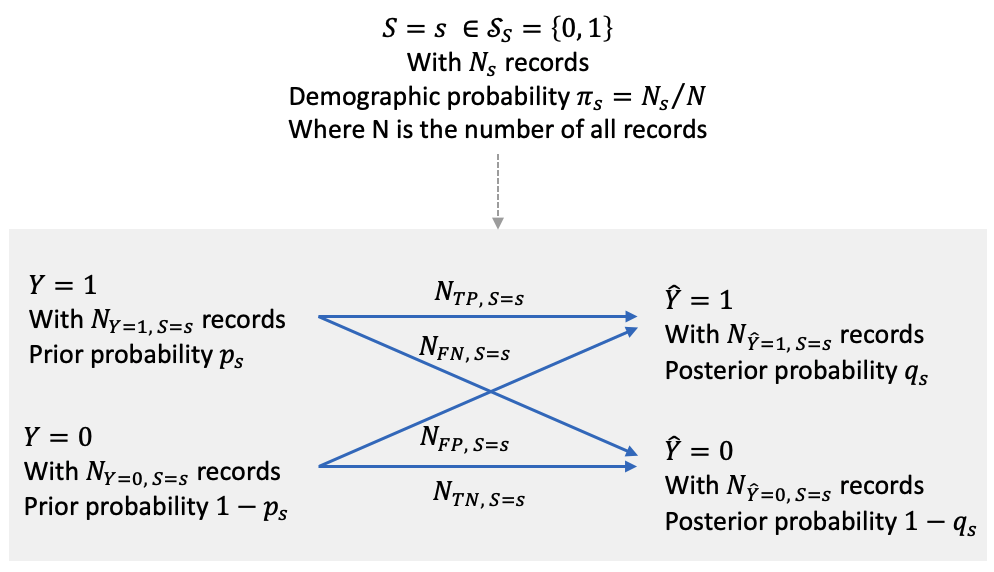}
\caption{An illustration of the main parameters of the binary channel model.}
\label{figBasicECs}
\end{figure}

\begin{table*}[!t]
\centering
\caption{Parameters of the binary channel model from various viewpoints per social group}
\label{tabParam}
\begin{tabular}{@{}m{0.08\textwidth} p{0.42\textwidth} p{0.45\textwidth}}
\toprule
Viewpoint	& Symbol/relation	& Terminology	\\
\midrule
Generic	& 	\begin{tabular}{@{}p{0.43\textwidth} p{0.45\textwidth}}
			$N_{TP, S=s}$																& 	Number of True Positives for group $s \in \{0, 1\}$ 			\\
			$N_{TN, S=s}$																& 	Number of True Negatives for group $s \in \{0, 1\}$			\\
			$N_{FP, S=s}$																& 	Number of False Positives for group $s \in \{0, 1\}$			\\
			$N_{FN, S=s}$																& 	Number of False Negatives for group $s \in \{0, 1\}$			\\
			$N_{s}= N_{S=s}= N_{TP, S=s}+N_{FN, S=s}+N_{FP, S=s}+N_{TN, S=s}$				& 	Number of all records for group $s \in \{0, 1\}$ 				\\
			$N_{t} = N_{S=0} + N_{S=1} = N_0 + N_1$										& 	Number of all records 	 							\\ 
			$ \pi_s = \pi_{S=s}= \frac{N_s}{N_t}$ (note:  $\pi_0 = 1 - \pi_1$) 						& 	Demographic-rate for group $s \in \{0, 1\}$ 				\\ 						
			\end{tabular} 																													\\
\midrule
Input	 space					& 	\begin{tabular}{@{}p{0.42\textwidth} p{0.46\textwidth}}
													$N_{Y=1, S=s}= N_{TP, S=s} + N_{FN, S=s}$									& Number of desired input classes for group $s \in \{0, 1\}$ 					\\
												 	$N_{Y=0, S=s} = N_s - N_{Y=1, S=s} = N_{FP, S=s} + N_{TN, S=s}$					& Number of undesired input classes for group $s \in \{0, 1\}$ 				\\
													$FPR_s = FPR_{S=s} = \frac{N_{FP, S=s}}{N_{Y=0, S=s}}$ 						& False Positive Rate for group $s \in \{0, 1\}$ 							\\
													$FNR_s = FNR_{S=s} = \frac{N_{FN, S=s}}{N_{Y=1, S=s}}$ 						& False Negative Rate for group $s \in \{0, 1\}$ 							\\
													$TPR_s = TPR_{S=s} = 1 - FNR_s = \frac{N_{TP, S=s}}{N_{Y=1, S=s}}$ 				& True Positive Rate for group $s \in \{0, 1\}$ 							\\
													$p_s = p_{S=s} = \frac{N_{Y=1, S=s}}{N_s}$ (and $1-p_s=\frac{N_{Y=0, S=s}}{N_s}$)	& Base-rate or prior probability for group $s \in \{0, 1\}$		\\
													\end{tabular} 																														\\
\midrule
Output space					& 	\begin{tabular}{@{}p{0.42\textwidth} p{0.46\textwidth}}
												 	$N_{\hat{Y}=1, S=s} = N_{TP, S=s} + N_{FP, S=s}$										& Number of desired output classes for group $s \in \{0, 1\}$ 			\\
													$N_{\hat{Y}=0, S=s}  = N_s - N_{\hat{Y}=1, S=s} = N_{FN, S=s} + N_{TN, S=s}$				& Number of undesired output classes for group $s \in \{0, 1\}$ 		\\
													$q_s = q_{S=s} = \frac{N_{\hat{Y}=1, S=s}}{N_s}$ (and $1-q_s = \frac{N_{\hat{Y} =0, S=s}}{N_s}$)	& Posterior probability for group $s \in \{0, 1\}$ 		\\ 																\end{tabular} 																														\\
\bottomrule
\end{tabular}
\end{table*}

\begin{example}
For the bank mortgage example, let's assume that the size (i.e., the number of all records) of a test dataset is $N_t=8000$. $N_t$ is the number of mortgage applications in the test dataset. The dataset contains the records of two social groups: The underprivileged group (i.e., $S=1$) of young workforce with $N_1 = 2000$ records and the privileged group (i.e., $S=0$) of experienced workforce with $N_0=6000$ records. From the unprivileged and privileged groups in the test dataset, 200 and 2000 applicants, respectively, have paid their mortgages back in due time. Thus, the base-rates for these groups, i.e., the within group fractions of those who successfully paid back their mortgages in due time, are $p_1 =\frac{N_{Y=1, S=1}}{N_1} = \frac{200}{2000} = 0.1$ and $p_0 =\frac{N_{Y=1, S=0}}{N_0} =  \frac{2000}{6000} = 0.33$. In this hypothetical example, we assume that the members of the privileged group (i.e., the experienced workforce) have paid back their mortgages in due time more often than the unprivileged group members of young workforces have done due to, for example, their relatively higher financial stability and income. The demographic-rates are $\pi_1 = \frac{N_1}{N_t} = \frac{2000}{8000} = 0.25$ and $\pi_0 = \frac{N_0}{N_t} = \frac{6000}{8000} = 0.75$.	

In Table \ref{tabIllustrativeExampleCM-OP3}, we present the values of the parameters of this case, named as Operation Point A, which are given or derived based on relations defined in Table \ref{tabParam}. 
$\Box$

\begin{table}[b]
\centering
\caption{A numeric instance for operation point A}
\label{tabIllustrativeExampleCM-OP3}
\begin{tabular}{c || c | c }
								&		group $S=0$			&	group $S=1$						\\	
\toprule
\multicolumn{3}{l}{Given parameter values}   																\\
\bottomrule
$N_t$								&	\multicolumn{2}{c}{8000} 						 					\\
\hline
$N_{s} $							&		6000					&	2000								\\
$N_{Y=1, S=s}$					&		2000					&	200								\\
$N_{TP, S=s} $						&		600					&	60								\\
$N_{FN, S=s} $						&		1400					&	140								\\
$N_{FP, S=s} $						&		1200					&	540								\\
$N_{TN, S=s} $						&		2800					&	1260								\\
\bottomrule
 \multicolumn{3}{l}{Calculated parameter values (based on relations defined in Table \ref{tabParam})}   					 \\
\toprule
$ \pi_s = \frac{N_s}{N_t}$				&	$\frac{6000}{8000} = 0.75$		&	$\frac{2000}{8000} = 0.25$		\\
$p_{s} =\frac{N_{Y=1, S=s}}{N_s}$		&	$\frac{2000}{6000} = 0.33$		&	$\frac{200}{2000} = 0.1$			\\
\bottomrule
\end{tabular}
\end{table}
\end{example}

As mentioned, in this study we assume two sensitive groups namely unprotected (or socially privileged) group $S=0$ and protected (or socially unprivileged) group $S=1$. For these groups we would like to study the relation between two formal fairness definitions called Statistical-Parity and Equalized-Odds. While the motivations for and detailed definitions of these measures will be provided in the following section, we summarize their frequentist probability definitions in Table \ref{tabParam} (see the rows corresponding to $FPR_s$, $TPR_s$ and $q_s$).

\subsection{Related Work} 
\label{subsec23} 

Several works give overviews about fairness metrics and measures from different perspectives, like \cite{bibVerma, binBalayn, mehrabi2021} about formal fairness measures in general and \cite{bibZafar, carey2023} about statistical fairness measures in particular. Concerning the scope of our paper, there is a rising number of studies that focus on the inconsistency among fairness measures and making trade-offs among them. For example, some works investigate the inherent trade-offs between formal metrics that quantify group and individual fairness like \cite{bibDwork, binns2020a}. Moreover, there are similar works to ours that investigate the incompatibility between some statistical fairness measures and argue about the need for making trade-offs among them. In the following, we review some of these works.   

Similar to our work, \cite{bibChouldechova} describes the conflict between two statistical fairness measures for binary classification when there is base-rate imbalance. The conflict between these fairness measures are surfaced after the COMPAS tool investigation. The work in \cite{bibKleinberg} extends the result of \cite{bibChouldechova} to a more generic case of soft output algorithmic decisions (i.e., risk scores). Both results of \cite{bibChouldechova, bibKleinberg} show the need for making trade-offs between contending measures. Making a tradeoff is investigated in, for example, \cite{bibReich2021}. Similarly, the study in \cite{hellman2020} reconsiders the legal and philosophical rationale behind the fairness measures that are shown to be in conflict in the case of the COMPAS tool. Unlike our work that concerns Equalized-Odds and Statistical-Parity measures in the presence of base-rate imbalance, the works in \cite{bibChouldechova, bibKleinberg, hellman2020} concerns two measures of Equalized-Odds and Predictive-Parity measures. 

In \cite{bibZhao2022}, it is argued that making an exact tradeoff between fairness and accuracy is not fully clear, even for classification problems. The paper investigates the inherent tradeoff between Statistical-Parity and accuracy in classification in presence of base-rate imbalance. The study provides a lower bound on accuracy for any fair classifier, meaning that when "the base-rates differ among groups, then any fair classifier satisfying Statistical-Parity has to incur a large error on at least one of the groups" \cite{bibZhao2022}. Our work is similar to \cite{bibZhao2022} but we consider a different pair of measures (i.e., Statistical-Parity and Equalized-Odds) and investigate how they interact, i.e., under which conditions they are compatible or not. Our choice of the pair of measures is motivated by the current practice and we extensively elaborate on how our results can impact the practice. In other words, our aim is to contribute to ongoing discourse, debate and practice by providing insight from technical discourse. To this end, we reconsider the non-ground-truth based Statistical-Parity measure of fairness through the lenses of ground-truth based Equalized-Odds measures, which are, in turn, based on FNR and FPR equality measures. The insight gained can be applicable at two levels namely when such ground truths are available and when they are not. 

The authors in \cite{bibGarg2020} use a mathematical framework to identify how different fairness measures relate to one another in being compatible or mutually exclusive. One of these relationships described in \cite{bibGarg2020}, like in our work, is between Statistical Parity and Equalized-Odds in presence of base-rate imballance. Unlike the mathematical framework in \cite{bibGarg2020}, our formalism is based on linear relations between these measures, which organically renders to a graphical representation of the relations between these measures. As such, our approach delivers a visual (and intuitive) sketch or representation of these relations, which can be used for an intuitive design or interpretation of the trade-offs between these measures. Further, unlike \cite{bibGarg2020}, we link our results to current legal frameworks and practices, especially in the area of selective labelling \cite{bibLakkaraju2017, bibArnold2025} and random preselections.

\section{Problem Description}
\label{sec4}
In this section, we present the theoretical basis of the study, upon which we motivate the problem investigated in this contribution. We provide the formal definitions and application areas of Statistical-Parity and Equalized-Odds in Sections \ref{subsec41} and \ref{subsec42}, respectively. We explain why establishing each of these fairness measures or both is important and illustrate their compatibility issues with a couple of examples. Note that we will investigate these compatibility issues analytically in Section \ref{sec5}. 

\subsection{Statistical-Parity Based Fairness}
\label{subsec41} 
One formal measure defined for dealing with disparate impact on groups is \textit{Statistical-Parity} \cite{bibDwork}; which is also referred to as equal acceptance rate \cite{bibZliobaite}, group fairness \cite{bibDwork}, demographic parity \cite{mehrabi2021}, and benchmarking \cite{bibSimoiu}. An advantage of this group-based measure is that it does not rely on ground-truth, i.e., (objectively) labeled target attributes, which is often scarce in practice. Statistical-Parity, see \cite{bibVerma, bibZafar}, can be defined as:

\begin{equation}
\label{eqStatiticalParity}
Pr(\hat{Y}=1 | \ S=1 ) = Pr (\hat{Y}=1 |\ S= 0),
\end{equation}
which, using the notation shown in Table \ref{tabParam}, can be written as 
\begin{equation}
\label{eqStatiticalParity2}
q_1 = q_0,
\end{equation}
because $q_0 = Pr (\hat{Y}=1 |\ S= 0)$ and $q_1 = Pr(\hat{Y}=1 | S=1)$.  

Establishing fairness according to Statistical-Parity defined in (\ref{eqStatiticalParity}) implies that individuals in both protected and unprotected groups are assigned to the positive predicted class with equal probability \cite{bibDwork}. Furthermore, the fact that an individual is assigned to the positive predicted class $\hat{Y}=1$ provides no information as to whether the individual is a member of which social group (i.e., whether the individual belongs to group $S=1$ or $S=0$). This definition imposes a fairness property to groups and as such it is a group fairness metric rather than an individual fairness metric. A related fairness measure to Statistical-Parity is the \textit{Calders-Verwer (CV) score} \cite{bibCalders} (and its generalization given in \cite{bibGalhotra}).

To indicate the applicability of Statistical-Parity measure, we elaborate on its relation to the so-called \emph{Demographic Population Representativity} \cite{bibClemmensen2022}. The Demographic Population Representativity measure is used in practice and recommended in some legal frameworks. For example, it is used in investigating whether students who received a Dutch government grant for their higher eduction were truly eligible for the grant, see  \cite{bibAlgorithmAudit2024}. This is a typical case of selective labelling \cite{bibLakkaraju2017, bibArnold2025}. As another example, Statistical-Parity is used in \cite{bibJong2005} to study minority overrepresentation in Dutch Criminal Justice System. An example legal framework that advocates using the measure is the anti-discrimination assessment framework of the Netherlands Institute for Human Rights, see \cite{bibCRM2025} page 24. Here, investigating Demographic Population Representativity is required in evaluating indirect discrimination in situations where profiling is used for preselecting individuals for detailed investigation. In the following, we adopt the way that the Demographic Population Representativity is applied into practice in report \cite{bibAlgorithmAudit2024} and show that, as such, it is equivalent to Statistical-Parity.

\begin{definition}
\label{defDemogParity}
\emph{Demographic Population Representativity} requires that the fraction of the number of positive predictions for a sensitive group to the total number of positive predictions (i.e., the outcomes for which $\hat{Y} = 1$ in the prediction space at the output of the model in Figure \ref{figBasicEC}) is the same as the fraction of the size of the sensitive group to the size of the population of interest (i.e., the number of the records in the ground-truth space at the input of the model in Figure \ref{figBasicEC}).
\end{definition}
Considering the definition of Demographic Population Representativity above, we can draw the following conclusion.
\begin{lemma}
\label{lemmaDemogPar}
The definition of Demographic Population Representativity (see Definition \ref{defDemogParity}) is the same as the definition of Statistical-Parity in (\ref{eqStatiticalParity}).
\end{lemma}
\begin{proof}
We start from having Demographic Population Representativity in place. Without loss of generality, let's consider the protected group $S=1$. Considering the notations in Figure \ref{figBasicECs}, the fraction of the protected group size to the population size is $\frac{N_1}{N_t}$ and the fraction of the number of the positive predictions for the protected group to that for all groups is  $\frac{q_1 N_1}{q_1 N_1 +  q_0 N_0}$. These two fractions should be equal according to Definition \ref{defDemogParity}. Considering that $N_s = \pi_s N_t$ for $s \in \{0, 1\}$ and $\pi_0 = 1 - \pi_1$, the equality of these fractions means that    
\begin{equation*}
\frac{N_1}{N_t} = \frac{q_1 N_1}{q_1 N_1 +  q_0 N_0} \ \ \rightarrow \ \ \frac{\pi_1 N_t}{N_t} =  \frac{q_1 \pi_1 N_t}{q_1 \pi_1 N_t +  q_0 (1-\pi_1) N_t},   
\end{equation*}
which, after simplification, can be written as
\begin{equation}
q_0 (1-\pi_1) = q_1  (1 - \pi_1) \quad  \rightarrow   \quad  q_0 = q_1. 
\label{eqDemParity4}
\end{equation}
The last relation in (\ref{eqDemParity4}) corresponds to the definition of Statistical-Parity defined in (\ref{eqStatiticalParity2}) - and thus in (\ref{eqStatiticalParity}). 

Conversely, starting from Statistical-Parity (i.e., $q_0 = q_1$), we can show that $\frac{q_1 N_1}{q_1 N_1 +  q_0 N_0} = \frac{q_1 N_1}{q_1 (N_1 +  N_0)} = \frac{N_1}{N_t}$, which is the Demographic Population Representativity in Definition \ref{defDemogParity}. Thus, Statistical-Parity and Demographic Population Representativity are equivalent.
\end{proof}

This means that the Demographic Population Representativity measure, as applied to the selective labelling use cases like that of \cite{bibAlgorithmAudit2024}, boils down to the Statistical-Parity measure. This is a testimony to the usability of Statistical-Parity and therefore from now on we focus on Statistical-Parity and investigate its relation to Equalized-Odds.

\subsection{Equalized-Odds Based Fairness}
\label{subsec42} 
Another formal measure defined for dealing with disparate treatment on groups is Equalized-Odds \cite{bibHardt}, also called conditional procedure accuracy equality \cite{binBerk} and disparate mistreatment \cite{bibZafar}. This group based measure is defined as: 
\begin{equation}
\label{eqEqualizedOdds}
Pr(\hat{Y}=1 | Y=y, S=1) =  Pr(\hat{Y}=1 | Y=y, S= 0), 		
\end{equation}
where $y \in \{0, 1 \}$. It is ground-truth based as it depends on also the input attribute $Y$ in the ground-truth space. Equalized-Odds is a combination of two fairness measures namely Predictive Equality and Equal Opportunity. Both measures provide an insight in fairness from the input viewpoint of the binary channel model shown in Figure \ref{figBasicEC}. 

The \emph{Predictive Equality} \cite{bibCorbett} is defined as:
\begin{equation}
\label{eqPredictiveEquality}
Pr(\hat{Y}=1 | Y=0, S=1 ) = \ Pr(\hat{Y}=1 | Y=0, S= 0), 
\end{equation}
which, using the notation shown in Table \ref{tabParam}, can be written as:
\begin{equation}
\label{eqPredictiveEquality2}
FPR_1 = FPR_0.
\end{equation}
The latter is because $FPR_0 = Pr(\hat{Y}=1 | Y=0, S= 0)$ and $FPR_1 = Pr(\hat{Y}=1 | Y=0, S=1 )$. We will use also the term $FPR$ equality to refer to Predictive Equality. 

The \emph{Equal Opportunity} \cite{bibHardt, bibKusner} is defined as: 
\begin{equation}
\label{eqEqualOpportunity}
Pr(\hat{Y}= 1 | Y=1, S=1)  =  Pr(\hat{Y}= 1 | Y=1, S=0), 		
\end{equation}
which, using the notation shown in Table \ref{tabParam}, can be written as :
\begin{equation}
\label{eqEqualOpportunity2}
TPR_1 = TPR_0.
\end{equation}
The latter is because $TPR_0 = Pr(\hat{Y}=1 | Y=1, S= 0)$ and $TPR_1 = Pr(\hat{Y}=1 | Y=1, S=1 )$. We will use the term $TPR$ equality to refer to Equal Opportunity. 

In this contribution we will also consider the equivalence of (\ref{eqEqualOpportunity2}), which is based on so-called $FNR$ equality because $TPR = 1 - FNR$. The resulting $FNR$ equality can be written as:
\begin{equation}
\label{eqErrorRateBal}
Pr(\hat{Y}= 0 | Y=1, S=1)  =  Pr(\hat{Y}= 0 | Y=1, S=0),		
\end{equation}
which can be written in terms of $FNR$s for the protected and unprotected groups - defined in Table \ref{tabParam} - as:
\begin{equation}
\label{eqErrorRateBal2}
FNR_1 = FNR_0.
\end{equation}

To indicate the applicability of Equalized-Odds measure, we elaborate on its usage areas. The $FPR$ and $TPR$ form the axes of the plane on which the Receiver Operating Characteristic (ROC) of classifiers can be specified. Therefore, the Equalized-Odds measure can give insight in fairness based on the (per group) ROC curve(s) of binary classifiers. The area under the ROC curve is an indication of the performance of a classifier relative to a random classifier. 

In addition to pair $FPR$ and $TPR$, our analysis will be presented with the equivalent pair of $FPR$ and $FNR$ on the $FPR-FNR$ plane. The combined $FPR$ and $FNR$ equalities, which is called \emph{Error Rate Balance} in \cite{bibChouldechova}, gives insight in the false prediction behavior of a binary classifier per sensitive group. False predictions, which can be of $FP$ and $FN$ types, can undesirably (i.e., either adversely or favorably) impact individuals who are falsely preselected and passed to the downstream stages of a decision-making process. (Not) going through the follow up stages can sometimes cause even long term harms on involved individuals. The impact of false predictions can be substantial when the follow-up stages of the decision-making process are heavily costly or highly favorable (e.g., the $FP$s and $FN$s, respectively, that occur in selective labelling cases). False predictions, therefore, can impact individuals and, at large, the society adversely, diminishing public trust in civil institutions who deploy such decision-making processes. 

\begin{example}
Following the bank mortgage example and based on the given parameters in Table \ref{tabIllustrativeExampleCM-OP3}, we present the calculated values of the parameters Statistical-Parity and Equalized-Odds for three Operation Points A, B and C in Table \ref{tabIllustrativeExampleCM-OP2}, where values of Operation A are from the previous example. The values for these operation points in Table \ref{tabIllustrativeExampleCM-OP2} are derived based on the relations in Table \ref{tabParam}.

\begin{sidewaystable}
\centering
\caption{A numeric instance for operation points A, B and C of the running example}
\label{tabIllustrativeExampleCM-OP2}
\begin{tabular}{c || c | c || c | c || c | c}
\toprule
								&		\multicolumn{2}{c}{operation Point A} 					& 	\multicolumn{2}{c}{operation Point B}   		& 	\multicolumn{2}{c}{operation Point C}  		\\
								&		group $S=0$			&	group $S=1$				&	group $S=0$		&	group $S=1$		&	group $S=0$		&	group $S=1$		\\	
\bottomrule
 \multicolumn{7}{l}{Given parameter values (for Operation Points A, B and C}   							 \\
\toprule
$N_{t}$							&	\multicolumn{2}{c||}{8000} 						 			& 	   			\multicolumn{2}{c||}{8000} 		& 	   			\multicolumn{2}{c}{8000} 		\\
$N_{TP, S=s} $						&		600					&	60						&		1400			&	140				&	1400			&	140				\\
$N_{FN, S=s} $						&		1400					&	140						&		600			&	60				&	600			&	60				\\
$N_{FP, S=s} $						&		1200					&	540						&		1200			&	540				&	400			&	460				\\
$N_{TN, S=s} $						&		2800					&	1260						&		2800			&	1260				&	3600			&	1340				\\
$N_{s} $							&		6000					&	2000						&		6000			&	2000				&	6000			&	2000				\\
$N_{Y=1, S=s}$					&		2000					&	200						&		2000			&	200				&	2000			&	200				\\
\bottomrule
 \multicolumn{7}{l}{Calculated parameter values (based on relations defined in Table \ref{tabParam})}   						 \\
\toprule
$ \pi_s = \frac{N_s}{N_t}$					&	$\frac{6000}{8000} = 0.75$		&	$\frac{2000}{8000} = 0.25$		&	$\frac{6000}{8000} = 0.75$			&	$\frac{2000}{8000} = 0.25$		&	$\frac{6000}{8000} = 0.75$			&	$\frac{2000}{8000} = 0.25$	\\
$p_{s} =\frac{N_{Y=1, S=s}}{N_s}$			&	$\frac{2000}{6000} = 0.33$		&	$\frac{200}{2000} = 0.1$			&	$\frac{2000}{6000} = 0.33$			&	$\frac{200}{2000} = 0.1$			&	$\frac{2000}{6000} = 0.33$			&	$\frac{200}{2000} = 0.1$		\\
\hline
$q_{s} =\frac{N_{\hat{Y}=1, S=s}}{N_s}$		&	$\frac{600+1200}{6000} = 0.3$		&	$\frac{60+540}{2000} = 0.3$		&	$\frac{1400+1200}{6000} = 0.43$		&	$\frac{140+540}{2000} = 0.33$		&	$\frac{1400+600}{6000} = 0.3$	&	$\frac{140+460}{2000} = 0.3$		\\
$FNR_s = \frac{N_{FN, S=s}}{N_{Y=1, S=s}}$	&	$\frac{1400}{2000} = 0.7$			&	$\frac{140}{200} = 0.7$			&	$\frac{600}{2000} = 0.3$				&	$\frac{60}{200} = 0.3$			&	$\frac{600}{2000} = 0.3$		&	$\frac{60}{200} = 0.3$		\\
$TPR_s = \frac{N_{TP, S=s}}{N_{Y=1, S=s}}$	&	$\frac{600}{2000} = 0.3$			&	$\frac{60}{200} = 0.3$			&	$\frac{1400}{2000} = 0.7$				&	$\frac{140}{200} = 0.7$			& 	$\frac{1400}{2000} = 0.7$		&	$\frac{140}{200} = 0.7$		\\
$FPR_s = \frac{N_{FP, S=s}}{N_{Y=0, S=s}}$	&	$\frac{1200}{6000-2000} = 0.3$		& 	$\frac{540}{2000-200} = 0.3$		&	$\frac{1200}{6000-2000} = 0.3$			& 	$\frac{540}{2000-200} = 0.3$		&	$\frac{400}{6000-2000} = 0.1$			& 	$\frac{460}{2000-200} = 0.2556$		\\
\bottomrule
\end{tabular}
\end{sidewaystable}

The Statistical-Parity requires that the ratio of positive predictions per group are the same, i.e., $q_0 = q_1$ - see (\ref{eqStatiticalParity2}). This means we have a classifier that produces positive outcomes (e.g., predicting that persons are able to pay back their mortgages in due time) with the same proportion for both privileged and unprivileged groups. Thus, looking at the outcome of the classifier, the ratio of the number of 1's per group to the size of the group is the same for both groups. For example, 30\% of either group is preselected at Operation Point A (i.e., assigned as trustworthy for on-time paying back their mortgages) and their mortgage applications are allowed to proceed to the following stages of the decision-making process. This measure does not rely on the values of the ground truth $Y$. 

Both Equalized-Odds measures, i.e., $FPR$ equality and $TPR$ equality - see (\ref{eqErrorRateBal2}) and (\ref{eqEqualOpportunity2}) - do rely on knowing the values of the ground truth $Y$. The $FPR$ equality captures the proportions of those in the test dataset who did not pay their mortgage depths in due time but are marked by the classifier otherwise (i.e., as being able to pay back their mortgage in due time). For example, at both Operation Points A and B, 30\% of either group is wrongly classified as trustworthy for on-time paying back their mortgages. For Operation Point B this leads to $q_0  > q_1$, which is a favoritism for the members of the privileged group from the Statistical-Parity viewpoint. At Operation Point C, the Statistical-Parity is preserved, but as a result $FPR_0 < FPR_1$, which is a favoritism for the members of the unprivileged group from the $FPR$ equality viewpoint. 

The $FNR=1-TPR$ captures the proportions of those in the test dataset who did pay their mortgage debts in due time but are marked by the classifier otherwise (i.e., as not being able to pay back their mortgage in due time). For example, at both Operation Point A and B, 70\% and 30\%, respectively, of either group is wrongly classified as untrustworthy for on-time paying back their mortgages. This is an unfair treatment of individuals whose mortgage applications are wrongfully disapproved by the classifier, thus not allowed to proceed to the following stages of the mortgage process. In this example, the impact on algorithmic disapproval is severe for individuals who are eligible for receiving mortgages. 

Let's assume that we have a classifier (or a pair of classifiers) for each of the Operation Points A, B and C in Table \ref{tabIllustrativeExampleCM-OP2}. The operation Point A, where all the three criteria (i.e., the $FPR$, the $FNR$ and the posterior probability $q$) are the same for both groups, can be seen as a fair point from the viewpoint of these three criteria. This does not hold for Operation Points B and C, where one of these measures does not hold. In situations where discrepancy in any of these measures is unacceptable, the classifiers corresponding to Operation Points B and C are unfair to use. Thus, when the base-rates are different, it is not always possible to establish all three fairness criteria among the sensitive groups. $\Box$
\end{example}

In this contribution, we analyze and give insight in when and how these fairness relations are (in)compatible. 

\section{(In)compatibility Analysis}
\label{sec5} 
In this section, we use our own approach to analyze the relations between, on the one hand, Statistical-Parity measure defined in (\ref{eqStatiticalParity} or \ref{eqStatiticalParity2}) and, on the other hand, the Equalized-Odds measures defined in Relations (\ref{eqPredictiveEquality} or \ref{eqPredictiveEquality2}) and (\ref{eqEqualOpportunity} or \ref{eqEqualOpportunity2}) $-$ or the equivalence of the latter defined in (\ref{eqErrorRateBal} or \ref{eqErrorRateBal2}). To this end, we start in Section \ref{subsec51} with assuming that the Equalized-Odds measures are in place and discuss how this would affect the Statistical-Parity measure for protected and unprotected groups. Subsequently, we investigate the opposite in Section \ref{subsec52} by assuming that the Statistical-Parity measure is in place and discuss how this would affect the Equalized-Odds measures for protected and unprotected groups. 

\subsection{Having Equalized-Odds in Place}
\label{subsec51} 

Let's assume that we have a training data set with objectively labeled class attribute $Y$ (i.e., having the values of attribute $Y$ in the ground-truth space).  With objectively labeled class attribute we mean that "historical data contain no discrimination"  \cite{bibZliobaite2017}. In some situations this assumption is realistic like in our running example where the label denotes whether individuals have actually paid back their loans in due time or not. This assumption, however, might be unrealistic in historical human decision makings, like making hiring decisions. In such cases the ground truth based fairness measures "should be considered with caution" \cite{bibZliobaite2017}. See \cite{binChoenni2025} for an elaborated typology of ground-truth types. 

Further, we assume that the training results in a classification model with the same Equalized-Odds measures (i.e., $FPR$s and $TPR$s) for both protected and unprotected groups. Considering (\ref{eqPredictiveEquality2}) and (\ref{eqEqualOpportunity2}), the latter assumption can be expressed as $FPR_1 = FPR_0$ and $TPR_1 = TPR_0$. 
The theorem below specifies when meeting Equalized-Odds measures results in meeting Statistical-Parity defined in (\ref{eqStatiticalParity2}) or (\ref{eqStatiticalParity}). 

\begin{theorem}
\label{theoremERB}
When both Equalized-Odds measures are established for protected and unprotected groups, i.e., $FPR_1 = FPR_0 = FPR_{\ast}$ and $TPR_1 = TPR_0= TPR_{\ast}$, then having also Statistical-Parity requires that 
\begin{equation}
\text{or} \quad \begin{cases*}
p_0 = p_1 & \\ 
TPR_{\ast} = FPR_{\ast}. &
\end{cases*}
\label{eqEqBaseRate}
\end{equation}
\end{theorem}

\begin{proof} 
Assume that the Equalized-Odds measures, see (\ref{eqPredictiveEquality2}) and (\ref{eqEqualOpportunity2}), are established for both protected and unprotected groups and denote their values by $FPR_{\ast} = FPR_1 = FPR_0$ and $TPR_{\ast} = TPR_1 = TPR_0$. Then, the probability of positive prediction outcomes (i.e., $\hat{Y}=1$) for group $S=0$ can be written based on the law of total probability as:
\begin{equation}
\begin{split}
&Pr(\hat{Y} = 1 | S=0 ) = 	\\
&Pr(Y= 1 | S=0 ) \cdot Pr(\hat{Y}=1   | Y=1, S=0) \ + \\
&Pr(Y= 0 | S=0 ) \cdot Pr(\hat{Y}=1 | Y=0, S=0).   
\end{split}
\label{eqQ01} 
\end{equation}
Substituting the equivalent parameters in Relation (\ref{eqQ01}) from the notations adopted in Table \ref{tabParam} and Figure \ref{figBasicECs}, the posterior probability for the outcome $\hat{Y}=1$ of group $S=0$ can be written as:
\begin{equation}
q_0  	=  p_0 TPR_{\ast} + (1 - p_0) FPR_{\ast}.
\label{eqQ0} 
\end{equation}
Similarly, for group $S=1$ we have: 
\begin{equation}
q_1  	=  p_1 TPR_{\ast} + (1 - p_1) FPR_{\ast}.
\label{eqQ1} 
\end{equation}
To establish also Statistical-Parity means that $q_0 = q_1$ according to (\ref{eqStatiticalParity2}). Substituting from (\ref{eqQ0}) and (\ref{eqQ1}) in $q_0 = q_1$ and after some manipulation we have: 
\begin{equation*}
(p_0 - p_1) \cdot (TPR_{\ast} -  FPR_{\ast}) = 0, 
\end{equation*}
which results in (\ref{eqEqBaseRate}) in order to have also Statistical-Parity satisfied.
\end{proof}

The top row in (\ref{eqEqBaseRate}) means having equal base-rate in place and the bottom row, which can be written as $1 - FNR_{\ast} = FPR_{\ast}$ or
\begin{equation}
FNR_{\ast} + FPR_{\ast} = 1,
\label{eqTradeOff} 
\end{equation}
means having a linear trade-off between $FPR$ and $FNR$. Either of these is an extra precondition needed for establishing Statistical-Parity. 

\begin{corollary}
\label{coroEqBRB}
When the bases rates are unequal, achieving Statistical-Parity (and thus Demographic Population Representativity) via establishing equal $FPR$s and $TPR$s for both protected and unprotected groups as indicated in the bottom row of (\ref{eqEqBaseRate}) (i.e., $FPR_1 = FPR_0 = TPR_1 = TPR_0 = x_{\ast}$) makes the posterior probability for outcome $\hat{Y}=1$ for each of these groups equal to $q_0 = q_1 = x_{\ast}$.    
\end{corollary}
\begin{proof}
Substituting the bottom row of (\ref{eqEqBaseRate}) in (\ref{eqQ0}) results in $q_0  =  p_0 TPR_0 + (1- p_0) FPR_0 = p_0 x_{\ast} + (1-p_0) x_{\ast} = x_{\ast}$. Similarly, substituting the bottom row of (\ref{eqEqBaseRate}) in (\ref{eqQ1}) results in $q_1 = x_{\ast}$. Thus, $q_0 = q_1 = x_{\ast}$.
\end{proof}
 
\subsection{Having Statistical-Parity in Place}
\label{subsec52} 

When there is no objectively labeled target attribute, a starting point for designing fair classifiers or for evaluating their fairness is to establish or check Statistical-Parity (i.e., Demographic Population Representativity as expressed in Lemma \ref{lemmaDemogPar}) because Statistical-Parity does not rely on ground-truth (i.e., values of $Y$). This situation arises in practice (arguably) more often than the case described in Section \ref{subsec51}, as we can see in practical cases \cite{bibAlgorithmAudit2024, bibJong2005}. A question that arises in such cases is: What would be the impact of the current practice, i.e., enforcing the Statistical-Parity which means $q_0 = q_1$, on the $FPR$s and $FNR$s for protected and unprotected groups, specially when their base-rates differ (i.e., $p_0 \neq p_1$)? 

As seen in (\ref{eqQ0}) and (\ref{eqQ1}), and based on the notation in Figure \ref{figBasicECs}, we can write: 
\begin{equation}
p_s \cdot TPR_s + (1-p_s) \cdot FPR_s = q_s,
\label{eqEqBaseRateS}
\end{equation}
for $s \in \{0, 1 \}$. In this section we have assumed that the Statistical-Parity is in place (i.e., $q_0 = q_1$). To simplify the notation, we use parameter $0 \leq q_{\ast} \leq 1$ to denote $q_{\ast} = q_0 = q_1$. Thus, (\ref{eqEqBaseRateS}) becomes:
\begin{equation}
\begin{split}
& p_s \cdot TPR_s + (1-p_s) \cdot FPR_s 	= q_{\ast}  \quad \text{or} \\
& TPR_s = (1 - \frac{1}{p_s}) FPR_s + \frac{q_{\ast}}{p_s}.
\end{split}
\label{eqLinearTradeOff} 
\end{equation}

Relation (\ref{eqLinearTradeOff}) represents a system of parallel lines in the $FPR-TPR$ plane, by varying the value of parameter $q_{\ast}$ in the range of $0 \leq q_{\ast} \leq 1$ . The slope of all these lines is equal to $1-\frac{1}{p_s}$, which is a negative value because $0<p_s<1$. The generic form of the binary classification performance lines, with an arbitrary $p_s$ and a varying $q_{\ast}$, is illustrated in Figure \ref{figGSlopes}.
\begin{figure}[htbp]
\centering
\includegraphics[scale=0.25]{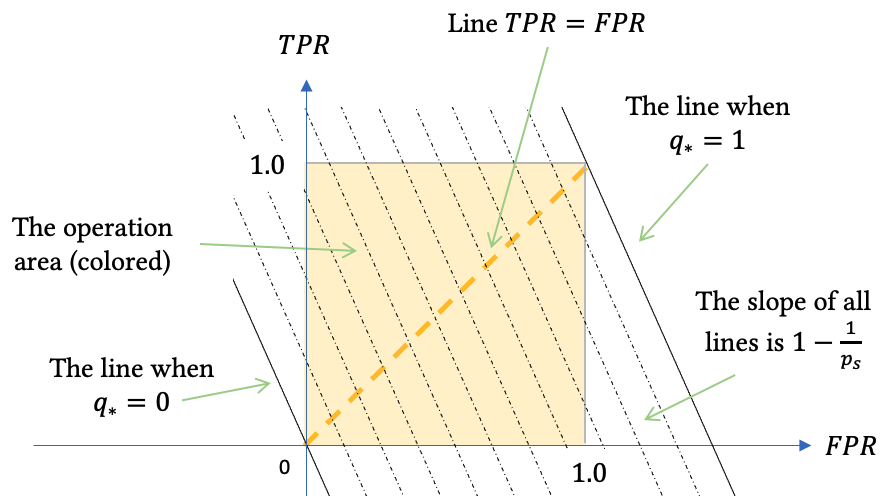}
\caption{An illustration of the generic form of the binary classification performance lines, with an arbitrary $p_s$ and a varying $q_{\ast}$.}
\label{figGSlopes}
\end{figure}

Relation (\ref{eqLinearTradeOff}) represents two linear equations and/or graphs (i.e., lines) for Equalized-Odds measure, corresponding to base-rates $p_0$ and $p_1$ in the $FPR-TPR$ plane. These lines can be given as: 
\begin{equation}
\label{eqLinearEquations}
  \left\{\begin{aligned} 
	 L_0: TPR = (1 - \frac{1}{p_0}) FPR + \frac{q_{\ast}}{p_0}  &  \  \text{for group $S=0$}, \\
   	 L_1: TPR = (1 - \frac{1}{p_1}) FPR + \frac{q_{\ast}}{p_1}  &  \  \text{for group $S=1$}.
	 \end{aligned}\right.
\end{equation}
The following theorem holds for lines $L_0$ and $L_1$.

\begin{theorem}
\label{theoremESP}
Assuming that Statistical-Parity is in place in a binary classifier (i.e., $ q_0 = q_1 = q_{\ast}$) and there is a base-rate imbalance for protected and unprotected groups  (i.e., $p_0 \neq p_1$), then the Equalized-Odds lines for these groups, see (\ref{eqLinearEquations}), cross each other at point $(q_{\ast}, q_{\ast})$ on the $FPR-TPR$ plane. 
\end{theorem}

\begin{proof}
The lines $L_0$ and $L_1$ in (\ref{eqLinearEquations}) cross each other (i.e., they are not parallel lines) because their slopes differ according to the assumption (i.e., $(1 - \frac{1}{p_0}) \neq  (1 - \frac{1}{p_1})$).  Let's assume lines $L_0$ and $L_1$ cross each other at point $(FPR_{\ast}, TPR_{\ast})$ on the the $FPR-TPR$ plane. Then, we can write ( \ref{eqLinearEquations}) as:
\begin{equation}
 \label{eqLinearEquations2}
  \left\{\begin{aligned} 
 	 L_0: TPR_{\ast} = (1 - \frac{1}{p_0}) FPR_{\ast} + \frac{q_{\ast}}{p_0},  &  \\
    	 L_1: TPR_{\ast} = (1 - \frac{1}{p_1}) FPR_{\ast} + \frac{q_{\ast}}{p_1}.  &
 	\end{aligned}\right.
 \end{equation}
To find the intersection point  $(FPR_{\ast}, TPR_{\ast})$, we eliminate variable $TPR_{\ast}$ in System (\ref{eqLinearEquations2}) to obtain the parametric value of $FPR_{\ast}$ as follows:
\begin{eqnarray}
 (1 - \frac{1}{p_0}) FPR_{\ast} + \frac{q_{\ast}}{p_0} 	& 	=  	&	(1 - \frac{1}{p_1}) FPR_{\ast} + \frac{q_{\ast}}{p_1},  \nonumber \\
(\frac{1}{p_1}-\frac{1}{p_0}) FPR_{\ast} 	& 	=  	&  \frac{q_{\ast}}{p_1} - \frac{q_{\ast}}{p_0},  \nonumber \\
FPR_{\ast} 	& 	=  	&  q_{\ast}. \label{eqFirstTerm}
\end{eqnarray}
Substituting the value of $FPR_{\ast}$ obtained in (\ref{eqFirstTerm}) in one of the lines in (\ref{eqLinearEquations}), we have
\begin{equation}
TPR_{\ast} = (1 - \frac{1}{p_0}) q_{\ast} +  \frac{q_{\ast}}{p_0} = q_{\ast}.  
\label{eqSecondTerm}
\end{equation}
From (\ref{eqFirstTerm}) and (\ref{eqSecondTerm}) we can conclude that $(q_{\ast}, q_{\ast})$ is the point at which the Equalized-Odds lines for protected and unprotected groups cross each other in the $FPR-TPR$ plane. 
\end{proof}

\begin{corollary}
\label{coroESP}
When Statistical-Parity is in place in a binary classifier and there is a base-rate imbalance for protected and unprotected groups, then the intersection point of the Equalized-Odds lines for these groups lie on the line $TPR = FPR$ on the $FPR-TPR$ plane. 
\end{corollary}
\begin{proof}
The coordinates of the intersection point are equal, see (\ref{eqFirstTerm}) and (\ref{eqSecondTerm}). Thus, the point lies on line $TPR = FPR$.
\end{proof}
Note that the intersection point lying on the line $TPR = FPR$ is the same as the second condition in (\ref{eqEqBaseRate}). 
 
\begin{example}
In Table \ref{tabIllustrativeExampleCM-OP2}, we derived the values of the $FPR_s$, $TPR_s$ and $q_s$, where $s \in \{0, 1 \}$, for operation Points A, B and C based on the relations defined in Table \ref{tabParam}. Here we derive the same values based on 
\begin{equation}
\begin{cases*}
q_0  	=  p_0 TPR_0 + (1 - p_0) FPR_0, & \\ 
q_1  	=  p_1 TPR_1 + (1 - p_1) FPR_1, &
\end{cases*}
\label{eqEqBaseRateGen}
\end{equation}
which are the extended forms of (\ref{eqEqBaseRateS}) or (\ref{eqLinearTradeOff}). The results are summarized in Table \ref{tabIllustrativeExampleSec} for these Operation Points as well as for a new operation Point D, which shows another case that two out of three measures are satisfied (without mentioning its detailed statistics for brevity). 

Operation Point A has the same $FPR$, $FNR$ and $q$ values for both groups, i.e., $FPR_0 = FPR_1$,  $TPR_0 = TPR_1$ and  $q_0 = q_1$. Further, relations $FPR_s = TPR_s = q_s$ hold per group $s \in \{0, 1 \}$, as prescribed by Theorem \ref{theoremERB}, Theorem \ref{theoremESP}, and Corollary \ref{coroEqBRB}. Operation Point A can be called fair when satisfying both Statistical-Parity and Equalized-Odds is required. 

For Operation Point B the $FPR$ and $TPR$ values are the same for both groups but they do not lie on the $FPR_s = TPR_s$ line (i.e.,  $FPR_s \neq TPR_s$ for each $s \in \{0, 1 \}$). Further, this point has different values for $q_0$ and $q_1$. 

Operation Points C and D (are forced to) have the same posterior probability for outcomes 1 (i.e., $q_0=q_1$), but they do not operate at the same $FPR$ or $TPR$, respectively, per sensitive group. As such, the $FPR_s$ and $TPR_s$ values do not lie on the $FPR_s=TPR_s$ line, for $s \in \{0, 1 \}$. Operation points C and D are evidence for a naive enforcement of Statistical-Parity, in cases where disparity of $FPR$ or $TPR$ is not tolerable. For these reasons, operation Points B, C and D can be called unfair when satisfying both Statistical-Parity and Equalized-Odds is required.  $\Box$
 
\begin{table*}[htbp]
\centering
\caption{A numeric instance for the running example}
\label{tabIllustrativeExampleSec}
\begin{tabular}{c c c | c c c | c c c | c c }
\toprule	
\multicolumn{2}{c}{base-rates}	 		&	&	\multicolumn{2}{c}{False Positive Rate}	&	&	\multicolumn{2}{c}{True Positive Rate}  		&	&	 \multicolumn{2}{c}{Posterior probability} 		\\
$p_0 = 0.33$		& $p_1= 0.1$		&	&	$FPR_0$		&	 $FPR_1$			&	&	$TPR_0$	& 	$TPR_1$					&	& 	 $q_0$		& 	$q_1$	\\		
\toprule
\multicolumn{2}{c}{Point A}			&	&	$0.3$   		& 	$0.3$			&	&	$0.3$   		& 	$0.3$				&	& 	$0.3$   		& 	$0.3$	\\	
\toprule
\multicolumn{2}{c}{Point B} 			&	&	$0.3$   		&	 $0.3$			&	&	$0.7$   		& 	$0.7$				&	&	$0.43$		& 	$0.34$	\\
\toprule
\multicolumn{2}{c}{Point C} 			&	&	$0.1$		& 	$0.2556$			&	&	$0.7$   		& 	$0.7$				&	&	$0.3$   		& 	$0.3$	\\
\toprule
\multicolumn{2}{c}{Point D} 			&	&	$0.3$   		& 	$0.3$			&	&	$0.45$		& 	$0.8$				&	&	$0.35$   		& 	$0.35$	\\		
\bottomrule
\end{tabular}
\end{table*}
\end{example}

\section{Discussion}
\label{sec6} 
In Section \ref{subsec61}, we elaborate on the results obtained. Subsequently, we explain the legal and practical implications of the results in Section \ref{subsec62}.  

\subsection{Elaboration Based on Graphical Illustrations}
\label{subsec61} 

In this section, for practical reasons, our discussion will be based on ROC curves. These curves lie on the $FPR-TPR$ plane, which we have used throughout the paper to capture the relations between Equalized-Odds measures and their performance lines as defined in (\ref{eqLinearEquations}). Let's assume that the classifier is not biased for sensitive social groups in that it has the same ROC curve for both groups. This means that the classifier operates independently of attribute $S$. Here we present four cases by choosing a (per group) operation point for the classifier, i.e., choosing the tuple ($FPR_{s}, TRP_{s}, q_{s})$. These cases illustrate our results mentioned in the previous sections.

In Case I, we choose two operation points on the $TPR=FPR$ line. Figure \ref{figGSlopes0} shows lines $L_0$ (with $p_0=0.33$) and $L_1$ (with $p_1=0.1$) for the same $q_{\ast}$ for both groups. This means that $FPR_{\ast}$ and $TPR_{\ast}$ are the same for both groups and further $q_{\ast} = FPR_{\ast} = TPR_{\ast}$. In this case, as illustrated in Figure \ref{figGSlopes0}, we have chosen two Operation Points A (corresponding to the solid lines) and A' (corresponding to the dotted lines), with $q_{\ast}=0.3$ and $q'_{\ast}=0.7$, respectively. Although in this case we obtain a fair classifier at each of Operation Points A and A' from the viewpoint of the three measures of $FPR$, $TPR$, and the posterior probability of positive outcomes $q$ (thus, from the viewpoints of Equalized-Odds and Statistical-Parity measures), the resulting classifier is not useful from prediction performance viewpoint because it corresponds to a random classifier that is specified by the ROC chance line $TPR=FPR$. 

Based on the illustration in Figure \ref{figGSlopes0}, we argue that, in presence of base-rate imbalance, pushing for Statistical-Parity (see the solid or dashed pair of lines in Figure \ref{figGSlopes0}) results in one of the following options: 
\begin{itemize}
\item A random classifier that preserves also the Equalized-Odds (i.e., $FPR_s = TPR_s = q_s$ where $q_0=q_1$) as illustrated above in Case I (i.e., at Operation Point A or A' in Figure \ref{figGSlopes0}), or 
\item A classifier that does not preserve Equalized-Odds (i.e, violating $FPR$ equality, $TPR$ equality or both) to be illustrated in the following.  
\end{itemize}
In the first option (i.e., having a random classifier), seeking for a high posterior probability for outcomes $\hat{Y}=1$ for sensitive groups (i.e., for high $q_0$ and $q_1$ values) should be done cautiously since it inflicts high $FPR$s for those sensitive groups according to Corollary \ref{coroEqBRB}. The second option mentioned above might be intriguing if satisfying both of Equalized-Odds measures is not necessary. This option asks for adopting a classifier that performs (well) above the ROC chance line. In the following we elaborate on a case for this option (see Case II below). 
\begin{figure}[htbp]
\centering
\includegraphics[scale=0.26]{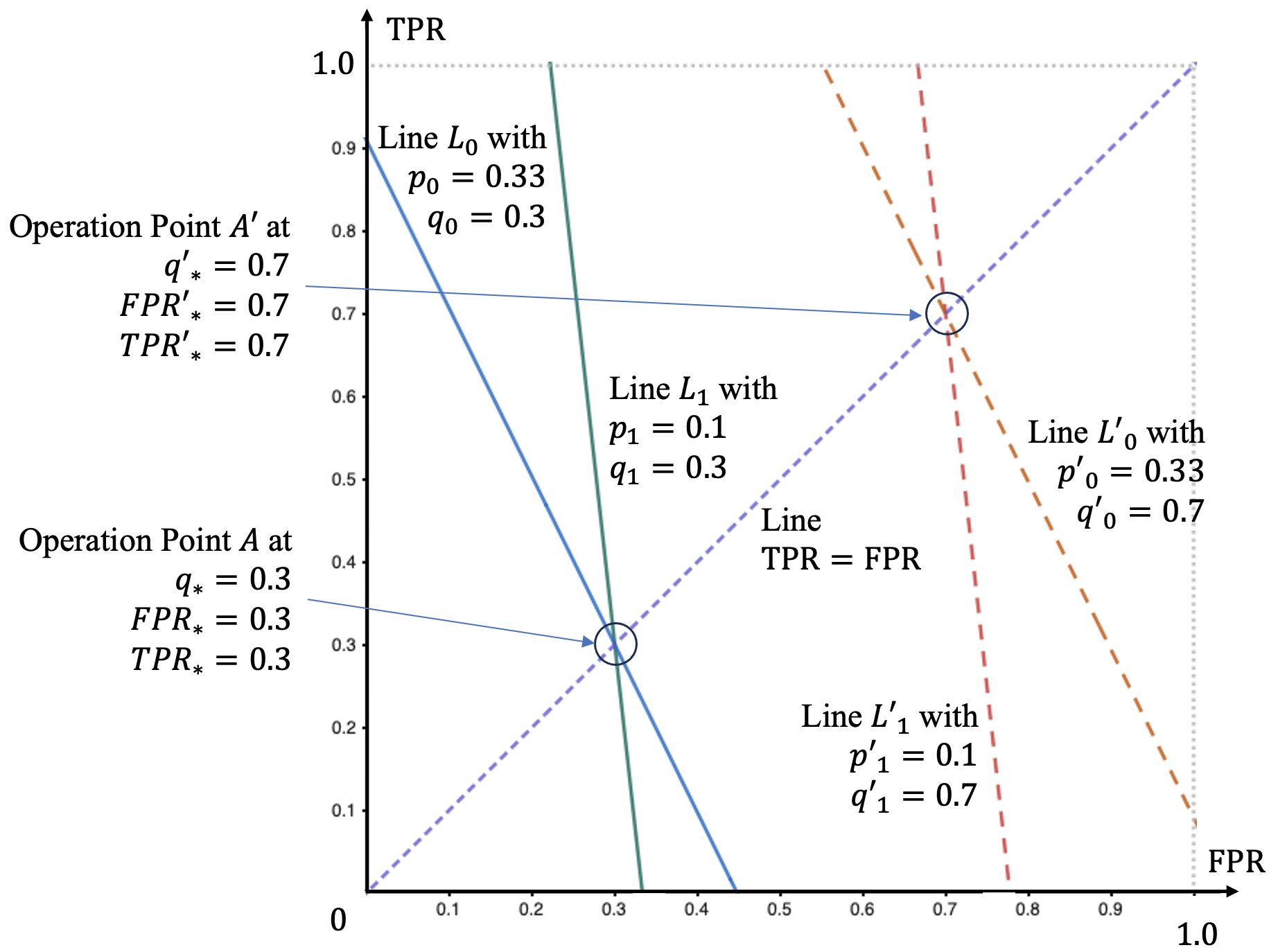}
\caption{Case I: Choosing the operation point on $TPR=FPR$ (i.e., on the ROC chance line).}
\label{figGSlopes0}
\end{figure}

When the binary classifier has a good ROC curve, i.e., the operation points lie well above the $TPR=FPR$ line, we present and explain three other illustrative cases shown in Figure \ref{figROCcurves}. In Case II, illustrated in Figure \ref{fig:sfig1}, we show that enforcing Statistical-Parity does not necessarily lead to having Equalized-Odds in place (i.e., having $FPR$ equality or $TPR$ equality) when there is base-rate imbalance. Here we enforce the same Statistical-Parity values for both groups (i.e., $q_0=q_1=0.3$), which means that we should seek the intersection points of lines $L_0$ and $L_1$ in Figure \ref{fig:sfig1} with the  ROC curve of the classifier. This results in two operation points, shown in Figure \ref{fig:sfig1}, one for group $S=0$ and and the other for group $S=1$. For these points hold $FPR_{1} \neq FPR_{0}$ and $TPR_{1} \neq TPR_{0}$ (and therefore with $FNR_{1} \neq FNR_{0}$), but with the same $q_0=q_1=0.3$. This case is possible if, for example, one chooses different thresholds for the sensitive groups. Considering the running example, the operation point for the unprivileged group appears to be more favorable than that for the privileged group, because it has higher $FPR$ and lower $FNR$. 

\begin{figure*}[htbp]
\begin{subfigure}{.33\textwidth}
  \centering
  \includegraphics[width=1.0\linewidth]{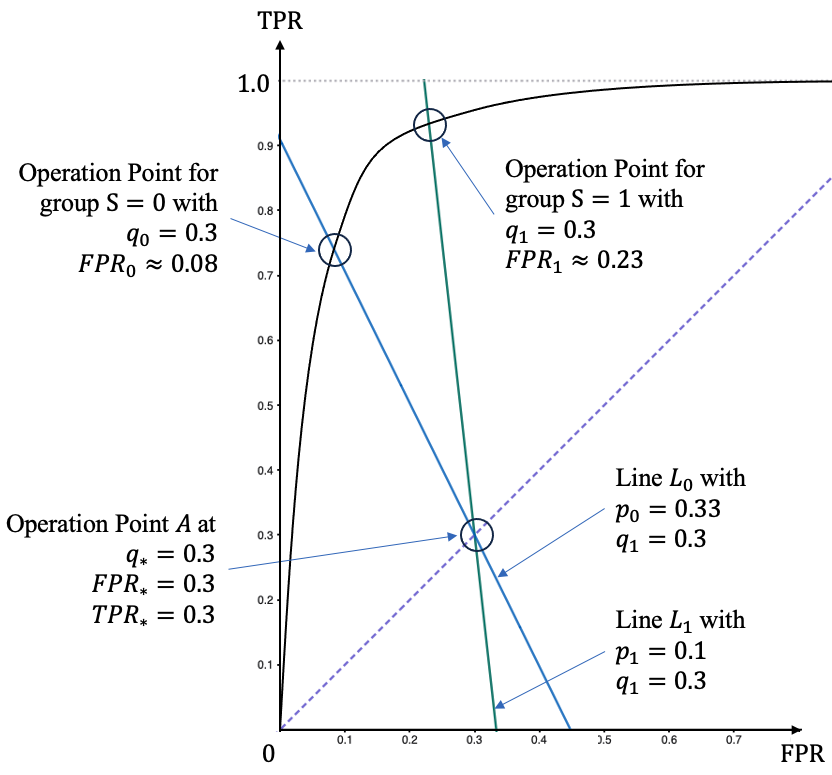}
  \caption{Case II}
  \label{fig:sfig1}
\end{subfigure}%
\begin{subfigure}{.33\textwidth}
  \centering
  \includegraphics[width=1.0\linewidth]{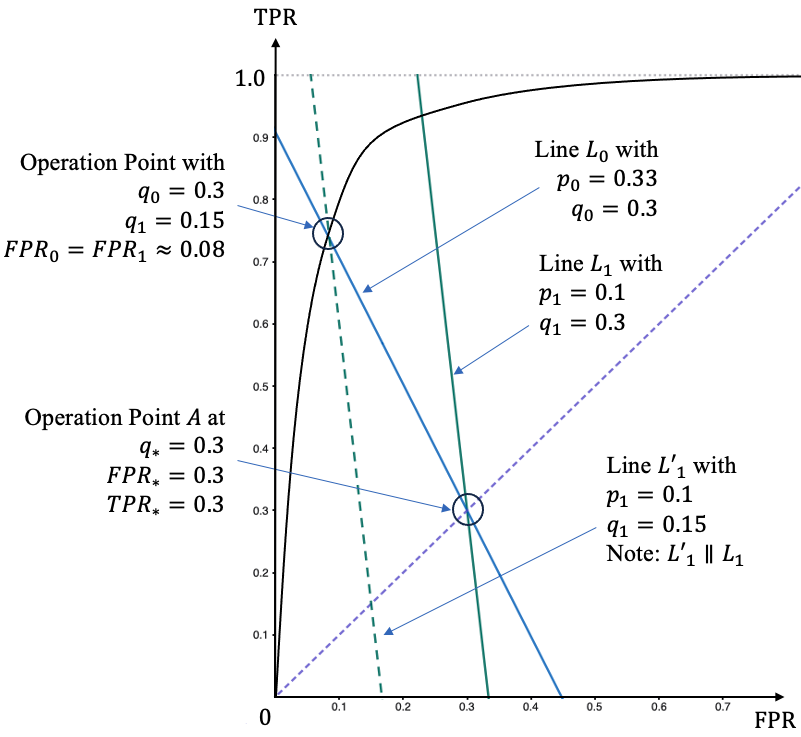}
  \caption{Case III}
  \label{fig:sfig2}
\end{subfigure}
\begin{subfigure}{.33\textwidth}
  \centering
  \includegraphics[width=1.0\linewidth]{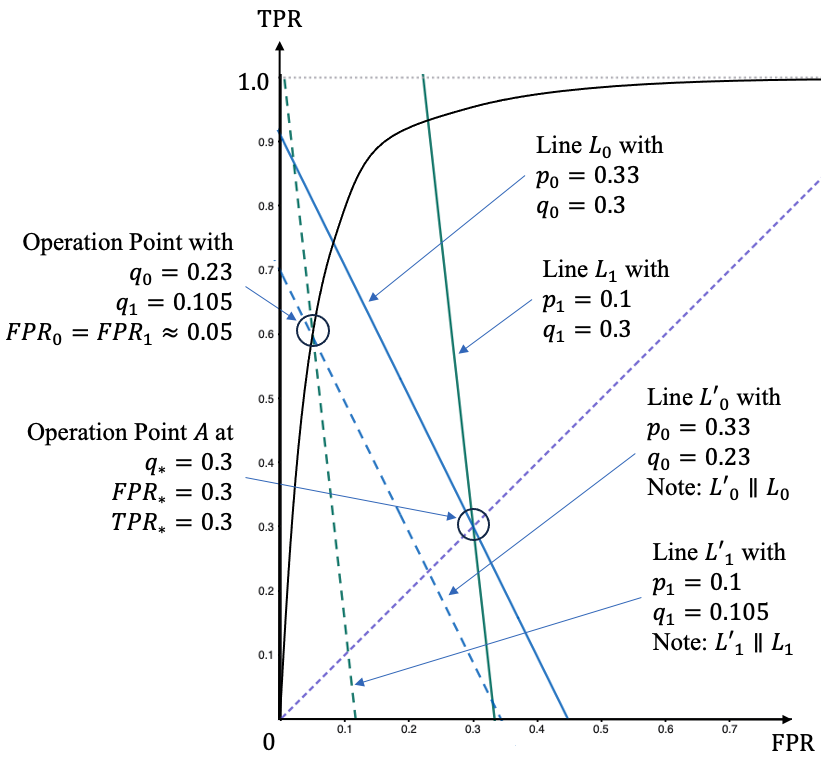}
  \caption{Case IV}
  \label{fig:sfig3}
\end{subfigure}%
\caption{An illustration of choosing the operation points on a good ROC curve.}
\label{figROCcurves}
\end{figure*}

From Theorem \ref{theoremESP} and Corollary \ref{coroESP} we can conclude that when we enforce having Statistical-Parity (thus Demographic Population Representativity) and assuming that the base-rates are different, then the chosen operation point with the same $FPR$ and $TPR$ for both groups may deliver different posterior probabilities for positive outcomes (i.e., $q_1 \neq q_0$) for sensitive groups. For example, see operation Point B in Table \ref{tabIllustrativeExampleSec}. This is because the operation point(s) of the classifier does (do) not lie on the $TPR = FPR$ line. Cases III and IV, shown in Figure \ref{figROCcurves}, are examples of this scenario. In Case III, illustrated in Figure \ref{fig:sfig2}, we choose the operation point at the intersection point of the ROC curve with line $L_0$ with $q_0=0.3$ (the same as that on Figure\ref{fig:sfig1}). This choice results in $FPR_{0} \approx 0.08$. As this is also the operation point chosen for group $S=1$, it must lie also on line $L'_1$ being parallel to $L_1$, for the rationale see (\ref{eqLinearTradeOff}). Therefore, the $q_1$ for $L_1$ and the $q_1$ for $L'_1$ differ, as indicated in Figure \ref{fig:sfig2}. Although here $FPR$ and $TPR$ equalities are established, the Statistical-Parity cannot be met, i.e., $q_0 \neq q_1$. Case IV in Figure \ref{fig:sfig3} shows a similar situation to that in Figure \ref{fig:sfig2}, but both lines $L'_0$ and $L'_1$ crossing the chosen operation point are parallel to those in Figure \ref{figGSlopes0}, again resulting in $q_0 \neq q_1$.   

One can take a step further than the cases shown in Figure \ref{figROCcurves} and, instead of having one classifier, use two different classifiers (i.e., with different ROC curves) for these groups. This scenario can be insightful when designing a pair of classifiers that preserves Statistical-Parity but do not preserve Equalized-Odds (i.e, violating $FPR$ equality, $TPR$ equality or both). We do not elaborate on this case here for brevity.  

\subsection{Practical Implications}
\label{subsec62}

As discussed in Section \ref{subsec61}, base-rate imbalance between sensitive social groups can lead to unfairness when fairness definition is based on Statistical-Parity and Equalized-Odds measures. Such a behavior is reported also for situations where fairness is defined based on Equalized-Odds and Predictive-Parity in \cite{bibChouldechova, bibKleinberg}. It is interesting to note that, as assumed in Section \ref{subsec61}, such an unfair behavior arises even if the binary classification model is independent of the sensitive attribute $S$. Thus, this type of unfair behavior is caused purely by the imbalance in the base-rate of (the data records of) the individuals for whom the classification model is instantiated to predict their class attribute $\hat{Y}$ - thus, it is a form of deployment bias. This is a sort of base-rate fallacy that holds for, in our case, Statistical-Parity and Equalized-Odds measures of fairness. According to some legal regimes like \cite{bibCRM2025}, a classifier can be assessed as fair if it does not differentiate based on a discrimination ground (i.e., with respect to sensitive attributes like race, gender and religion directly (or even indirectly). Such an assessment is incomplete if it is carried only for classification model training and testing. As shown here, one should also assess the performance of trained (supposedly non-discriminating) models based on whether the data records to which the model is applied has base-rate imbalance or not. 

It is important to investigate which formal definitions of fairness apply to a given usage context and to what extent \cite{bibBargh2025}. In Section \ref{sec4}, we explained the motivations for applying Statistical-Parity and Equalized-Odds measures to a usage context. Should applying both of these measures be necessary in a given usage context, then pursuing only one of these measures could violate the other one as we discussed in Section \ref{subsec61}. In some legal frameworks like \cite{bibCRM2025}, which aims at assessing the discrimination behavior of profiling algorithms, applying Statistical-Parity is directly recommended but there is no explicit mentioning of $FPR$ and $FNR$ related measures like Equalized-Odds. Not enforcing the latter can cause unfairness in situations where $FP$'s and $FN$'s are impactful.   

In practice there might be no data available about the real values of the class attribute $Y$. In such cases, therefore, many studies consider only Statistical-Parity when investigating fairness at the output of a classifier, like in \cite{bibAlgorithmAudit2024, bibJong2005}. From our study we can conclude that such studies cannot reveal $FPR$ and $FNR$ inequalities even if they show that Statistical-Parity is in place (see, e.g., Case II in Figure \ref{figROCcurves}). Conversely, if such studies show lack of Statistical-Parity, there might be $FPR$ and $FNR$ equalities in place (see, e.g., Cases III and IV in Figure \ref{figROCcurves}). If establishing only $FPR$ and $FNR$ equalities is required in a usage context (like in judicial judgements), then establishing Statistical-Parity is not relevant and the results of the studies that assess fairness based on only Statistical-Parity should be interpreted cautiously. As such, the results of this study can be insightful even for cases where we do not have the ground truth values of class attributes. The awareness about incompatibility of Statistical-Parity with Equalized-Odds can drive researchers to look for or estimate the possibility of base-rate imbalance as it could help correct interpretation of Statistical-Parity based assessments. 

There is no incompatibility between Statistical-Parity and Equalized-Odds measures if there is base-rate balance, i.e., $p_0 = p_1$, as shown in (\ref{eqEqBaseRate}). Therefore, to deal with base-rate fallacy in regard to Statistical-Parity and Equalized-Odds measures, a strategy would be to eliminate base-rate imbalance at the societal level by addressing the root cause of such imbalances, as also mentioned in \cite{bibChouldechova}. Otherwise, when establishing both Statistical-Parity and Equalized-Odds measures is necessary, as shown in this study, the solution boils down to establishing a random selection.

\section{Conclusion}
 \label{sec7} 

Justice and fairness principles that apply to data and data analytics at various stages of data journey can (partly) be specified by a set of formal measures like Statistical-Parity, Equalized-Odds, and Predictive-Parity. Knowing which fairness measures are relevant in a given context and which incompatibility forms exist among these measures is essential to monitor and preserve data fairness at various stages of data journey. Having insight in such incompatibilities enables making trade-offs among contending fairness measures and/or using complementary mitigations. 
 
In this contribution, we studied the impact of the imbalance between the base-rates of sensitive social groups when Statistical-Parity and Equalized-Odds must be satisfied for fair binary classification. We showed that enforcing both Statistical-Parity and Equalized-Odds requires either having base-rate balance or adopting a random classifier. Our analysis showed adopting an efficient classifier, which shows a good performance, e.g., higher accuracy, requires making trade-off between Statistical-Parity and Equalized-Odds. Our approach, which represents the interaction between these measures graphically, provides a visualization of the trade-offs occurring when design parameters are modified. 

Based on the insights gained, we recommended some adjustment to the current legal frameworks and practices which are based on using Statistical-Parity for assessing the fairness of datasets, classifiers and decision-making processes. For example, when enforcing both Statistical-Parity and Equalized-Odds is necessary and there is base-rate imbalance, then one must use a random classifier. As another outcome, assessing fairness based on Statistical-Parity, although being affordable in situations where there is no ground truth available for the class attribute, should be done cautiously if there is base-rate imbalance. We noted that having Statistical-Parity may be accompanied by $FPR$ or $FNR$ inequality; and lack of Statistical-Parity may be accompanied by $FPR$ equality and/or $FNR$ equality. Therefore, one should also investigate base-rate (im)balance for interpreting the outcome of Statistical-Parity based assessment of fairness.  

Several directions and avenues can be foreseen for further research. One direction is to extend the results to non-binary classification scenarios where the sensitive attribute and/or the target attribute is categorical or the target attribute is a soft output. As another direction, guidelines can be developed for making trade-offs between Statistical-Parity and Equalized-Odds measures. Our work and the works in \cite{bibChouldechova, bibKleinberg, bibGarg2020} motivate further research on and specification of the incompatibility forms that exist among other statistical fairness measures.  

\section*{Endmatter Section}

{\bf Generative AI usage statement.} The authors declare that they did not use Generative AI (GenAI), Large Language Models (LLMs), or any GenAI/LLM tool for generating text, assisting with formatting, assisting with grammar or improving fluency of writing. The authors utilized only a freely available Latex editor (TexShop) for spellchecking.

{\bf Disclaimer.} This preprint comprises results of a study carried out by the authors. Sharing the preprint via arXiv does not mean that its contents reflect the viewpoint of the WODC Research and Data Centre or the Dutch Ministry of Justice and Security.

\bibliographystyle{unsrt}
\bibliography{bib-Fairness}

\end{document}